\documentclass{article}

\usepackage{microtype}
\usepackage{graphicx}
\usepackage{subfigure}
\usepackage{float}
\usepackage{booktabs} 
\usepackage{hyperref}
\usepackage[accepted]{icml2019}
\usepackage{breakurl}
\usepackage{dsfont}
\usepackage{enumitem}
\usepackage{amsmath,amssymb,amsthm}
\usepackage{thmtools}
\usepackage{thm-restate}
\usepackage{xspace}
\usepackage{cleveref}
\usepackage{natbib}
\setcitestyle{authoryear, square}
\usepackage{marvosym}
\usepackage{fontawesome}
\usepackage{amssymb}
\usepackage{pifont}
\usepackage{stmaryrd,scalerel}

\newcommand{\cmark}{\ding{51}}%
\newcommand{\xmark}{\ding{55}}%
\usepackage{halloweenmath}
\newtheorem{theorem}{Theorem}
\newtheorem{lemma}{Lemma}

\newtheorem{example}{Example}
\newtheorem{remark}{Remark}
\newcommand{\Cc}{\mathsf{C}}
\newcommand{\Pc}{\mathsf{P}}
\newcommand{\Pf}{\mathfrak{P}}
\newcommand{\SOS}{\mathsf{SOS}}
\newcommand{\Ic}{\mathsf{I}}
\newcommand{\Fc}{\mathcal{F}}
\newcommand{\KL}{\mathsf{KL}}
\newcommand{\Zsf}{\mathsf{Z}}
\newcommand{\Xsf}{\mathsf{X}}
\newcommand{\av}{\mathbf{a}}
\newcommand{\xv}{\mathbf{x}}
\newcommand{\wv}{\mathbf{w}}
\newcommand{\rmd}{\mathrm{d}}
\newcommand{\zv}{\mathbf{z}}
\newcommand{\RR}{\mathds{R}}
\newcommand{\NN}{\mathds{N}}

\newcommand{\Nc}{\mathcal{N}}

\newcommand{\eg}{{e.g.}\xspace}
\newcommand{\vs}{{vs.}\xspace}
\newcommand{\ie}{{i.e.}\xspace}
\newcommand{\iid}{{i.i.d.}\xspace}

\newcommand{\wrt}{{w.r.t.}\xspace}

\newcommand{\Tb}{\mathbf{T}}
\newcommand{\Db}{\mathbf{D}}
\newcommand{\Lb}{\mathbf{L}}
\newcommand{\Tpush}{\Tb_{\#}}
\newcommand{\intr}{\mathop{\mathrm{int}}}
\newcommand{\cl}{\mathop{\mathrm{cl}}}

\newif\ifrev

\newcommand{\red}[1]{\textcolor{red}{#1}}

\newcommand{\id}{\mathop{id}}
\newcommand{\thetav}{\boldsymbol{\theta}}
\DeclareMathOperator\erf{erf}

\let\etoolboxforlistloop\forlistloop 
\usepackage{autonum}
\let\forlistloop\etoolboxforlistloop 
\usepackage{tikz}
\usetikzlibrary{positioning,calc,fit}
\usepackage{pgfplots}
\pgfplotsset{compat=1.14}

\icmltitlerunning{Sum-of-Squares Polynomial Flow}

\begin{document}

\twocolumn[
\icmltitle{Sum-of-Squares Polynomial Flow}




\begin{icmlauthorlist}
\icmlauthor{Priyank Jaini}{dc,waii,vec}
\icmlauthor{Kira A. Selby}{dc,waii}
\icmlauthor{Yaoliang Yu}{dc,waii}
\end{icmlauthorlist}

\icmlaffiliation{dc}{University of Waterloo, Waterloo, Canada}
\icmlaffiliation{waii}{Waterloo AI Institute, Waterloo, Canada}
\icmlaffiliation{vec}{Vector Institute, Toronto, Canada}

\icmlcorrespondingauthor{Yaoliang Yu}{yaoliang.yu@uwaterloo.ca}

\icmlkeywords{Triangular map, density estimation, universality, sum-of-squares}

\vskip 0.3in
]


\printAffiliationsAndNotice{}  

\begin{abstract}
Triangular map is a recent construct in probability theory that allows one to transform any source probability density function to any target density function. 
Based on triangular maps, we propose a general framework for high-dimensional density estimation, by specifying one-dimensional transformations (equivalently conditional densities) and appropriate conditioner networks. This framework (a) reveals the commonalities and differences of existing autoregressive and flow based methods, (b) allows a unified understanding of the limitations and representation power of these recent approaches and, (c) motivates us to uncover a new Sum-of-Squares (SOS) flow that is interpretable, universal, and easy to train. We perform several synthetic experiments on various density geometries to demonstrate the benefits (and short-comings) of such transformations. SOS flows achieve competitive results in simulations and several real-world datasets. 

\end{abstract}
\section{Introduction}
\label{sec:intro}

Neural density estimation methods are gaining popularity for the task of multivariate density estimation in machine learning \cite{KingmaSJCSW16, DinhKB15, DinhSDB17, PapamakariosPM17, UriaCGML16, HuangKLC18}. These generative models provide a tractable way to evaluate the exact density, unlike generative adversarial nets \cite{GoodfellowPMXFOCB14} or variational autoencoders \cite{KingmaWelling13, RezendeMW14}. Popular methods for neural density estimation are \emph{autoregressive models} \cite{Neal92,BengioBengio99,LarochelleMurray11,UriaCGML16} and \emph{normalizing flows} \cite{RezendeMohamed15,TabakVE10, TabakTurner13}. These models aim to learn an invertible, bijective and increasing transformation $\Tb$ that pushes forward a (simple) source probability density (or measure, in general) to a target density such that computing the inverse $\Tb^{-1}$ and the Jacobian $|\Tb'|$ is \emph{easy}.

 In probability theory, it has been rigorously proven that  increasing \emph{triangular} maps \cite{BogachevKM05} are universal, \ie any source density can be transformed into a target density using an increasing triangular map. Indeed, the Knothe-Rosenblatt transformation \citep[Ch.1,][]{Villani08} gives a (heuristic) construction of such a map, which is unique up to null sets \cite{BogachevKM05}. Furthermore, by definition the inverse and the Jacobian of a triangular map can be very efficiently computed through univariate operations. However, for multivariate densities computing the exact Knothe-Rosenblatt transform itself is not possible in practice. Thus, a natural question is: Given a pair of densities, how can we efficiently estimate this unique increasing triangular map? 
 
This work is devoted to studying these increasing, bijective, and monotonic triangular maps, in particular how to estimate them in practice. In \S\ref{sec:bg}, we precisely formulate the density estimation problem and propose a general maximum likelihood framework for estimating densities using triangular maps. We also explore the properties of the triangular map required to push a source density to a target density.  

Subsequently, in \S\ref{sec:prev}, we trace back the origins of the triangular map and connect it to many recent works on generative modelling. We relate our study of increasing, bijective, triangular maps to works on iterative Gaussianization \citep{ChenGopinath01,LaparraCVM11} and normalizing flows \citet{TabakVE10, TabakTurner13, RezendeMohamed15}. We show that a triangular map can be decomposed into compositions of one dimensional transformations or equivalently univariate conditional densities, allowing us to demonstrate that all \emph{autoregressive models} and \emph{normalizing flows} are subsumed in our general density estimation framework. As a by-product, this framework also reveals that \emph{autoregressive models} and \emph{normalizing flows} are in fact equivalent.  Using this unified framework, we study the commonalities and differences of the various aforementioned models. Most importantly, this framework allows us to study the universality in a much cleaner and more streamlined way. We present a unified understanding of the limitations and representation power of these approaches, summarized concisely in \Cref{tab: flows} below.

In \S\ref{sec:sos}, by understanding the pivotal properties of triangular maps and using our proposed framework, we uncover a new neural density estimation procedure called the Sum-of-Squares polynomial flows (SOS flows). We show that SOS flows are akin to higher order approximation of $\Tb$ depending on the degree of the polynomials used. Subsequently, we show that SOS flows are universal, \ie given enough model complexity, they can approximate any target density. We further show that (a) SOS flows are a strict generalization of the inverse autoregressive flow (IAF) of \citet{KingmaSJCSW16}, (b) they are interpretable; its coefficients directly control the higher order moments of the target density and, (c) SOS flows are easy to train; unlike NAFs \cite{HuangKLC18} which require non-negative weights, there are no constraints on the parameters of SOS.  

In \S\ref{sec:exp}, we report our empirical analysis. We performed holistic synthetic experiments to gain intuitive understanding of triangular maps and SOS flows in particular. Additionally, we compare SOS flows to previous neural density estimation methods on real-world datasets where it achieved competitive performance.

We summarize our main contributions as follows:
\begin{itemize}[itemsep=0pt, topsep=1pt]
    \item We study and propose a rigorous framework for using triangular maps for density estimation 
    \item Using this framework, we study the similarities and differences of existing flow based and autoregressive models
    \item We provide a unified understanding of the limitations and representational power of these methods
    \item We propose SOS flows that are universal, interpretable, and easy to train.
    \item We perform several synthetic and real-world experiments to demonstrate the efficacy of SOS flows. 
\end{itemize}
\vspace{-1em}
\section{Density estimation through triangular map}
\label{sec:bg}

In this section we set up our main problem, introduce key definitions and notations, and formulate the general approach to estimate density functions using triangular maps.


Let $p, q$ be two probability density\footnote{All of our results can be extended to two probability measures satisfying mild regularity conditions. For simplicity and concreteness we restrict to probability densities here.} functions (\wrt the Lebesgue measure) over the source domain $\Zsf\subseteq \RR^d$ and the target domain $\Xsf \subseteq\RR^d$, respectively. Our main goal is to find a \emph{deterministic} transformation $\Tb: \Zsf \to \Xsf$ such that for all (measurable) set $B \subseteq \Xsf$,
\begin{align}
\int_B q(\xv) \rmd\xv \approx \int_{\Tb^{-1}(B)} p(\zv) \rmd\zv.
\end{align}
In particular, when $\Tb$ is bijective and differentiable \citep[\eg][]{Rudin87}, we have the change-of-variable formula $\xv = \Tb(\zv)$ such that
\begin{align}
q(\xv) &= p(\zv) / |\Tb'(\zv)| \\
\label{eq:pf}&= p(\Tb^{-1}\xv) / |\Tb'(\Tb^{-1}\xv)|=: \Tpush p,
\end{align}
where $|\Tb'(\zv)|$ is the (absolute value) of the Jacobian (determinant of the derivative) of $\Tb$. In other words, by pushing the source random variable $\zv \sim p$ through the map $\Tb$ we can obtain a new random variable $\xv \sim q$. This ``push-forward'' idea has played an important role in optimal transport theory \cite{Villani08} and in recent Monte carlo simulations \cite{MarzoukMPS16,ParnoMarzouk18,PeherstorferMarzouk18}. 

Here, our interest is to learn the target density $q$ through the map $\Tb$. Let $\Fc$ be a class of mappings and use the KL divergence\footnote{Other statistical divergences can be used as well.} to measure closeness between densities. We can formulate the density estimation problem as:
\begin{align}
\label{eq:KL}
\min_{\Tb \in \Fc}~ \KL(q \| \Tpush p) \equiv -\int q(\xv) \log \frac{p(\Tb^{-1}\xv)}{|\Tb'(\Tb^{-1}\xv)|} \rmd \xv.
\end{align}
When we only have access to an \iid sample $\lbag \xv_1, \ldots, \xv_n \rbag\sim q$, we can replace the integral above with empirical averages, which amounts to maximum likelihood estimation: 
\begin{align}
\label{eq:ML}
\max_{\Tb \in \Fc}~~ \frac{1}{n}\sum_{i=1}^n \Big[-\log |\Tb'(\Tb^{-1}\xv_i)| + \log p(\Tb^{-1}\xv_i) \Big].
\end{align}
Conveniently, we can choose \emph{any} source density $p$ to facilitate estimation. Typical choices include the standard normal density on $\Zsf = \RR^d$ (with zero mean and identity covariance) and uniform density over the cube $\Zsf = [0,1]^d$.

Computationally, being able to solve \eqref{eq:ML} efficiently relies on choosing a map $\Tb$ whose
\begin{itemize}[itemsep=0pt,topsep=0pt]
	\item inverse $\Tb^{-1}$ is ``cheap'' to compute;
	\item Jacobian $|\Tb'|$ is ``cheap'' to compute.
\end{itemize}
Fortunately, this is always possible. Following \citet{BogachevKM05} we call a (vector-valued) mapping $\Tb: \RR^d \to \RR^d$ triangular if for all $j$, its $j$-th component $T_j$ only depends on the first $j$ variables $x_1, \ldots, x_j$. The name ``triangular'' is derived from the fact that the derivative of $\Tb$ is a triangular matrix function\footnote{The converse is clearly also true if our domain is  connected.}. We call $\Tb$ (strictly) increasing if for all $j\in[d]$, $T_j$ is (strictly) increasing \wrt the $j$-th variable $x_j$ when other variables are fixed.
\begin{theorem}[\citealt{BogachevKM05}]
	\label{thm:tri}
	For any two densities $p$ and $q$ over $\Zsf = \Xsf=\RR^d$, there exists a unique (up to null sets of $p$) \emph{increasing} triangular map $\Tb:\Zsf \to \Xsf$ so that $q = \Tpush p$. The same\footnote{More generally on any open or closed subset of $\RR^d$ if we interpret the monotonicity of $\Tb$ appropriately \citep{Alexandrova06}.} holds over $\Zsf = \Xsf = [0,1]^d$. 
\end{theorem}
Conveniently, to compute the Jacobian of an increasing triangular map we need only multiply $d$ \emph{univariate} partial derivatives 
$
|\Tb'(\xv)| = \prod_{j=1}^d \frac{\partial T_j}{\partial x_j}. 
$
Similarly, inverting an increasing triangular map requires inverting $d$ \emph{univariate} functions sequentially, each of which can be efficiently done through say bisection. \citet{BogachevKM05} further proved that the change-of-variable formula \eqref{eq:pf} holds for any increasing triangular map $\Tb$ (without any additional assumption but using the right-side derivative).

Thus, triangular mappings form a very appealing function class for us to learn a target density as formulated in \eqref{eq:KL} and \eqref{eq:ML}. Indeed, \citet{MoselhyMarzouk12} already promoted a similar idea for Bayesian posterior inference and \citet{SpantiniBM18} related the sparsity of a triangular map with (conditional) independencies of the target density. Moreover, many recent generative models in machine learning are precisely special cases of this approach. Before we discuss these connections, let us give some examples to help understand \Cref{thm:tri}.

\begin{example}
Consider two probability densities $p$ and $q$ on the real line $\RR$, with distribution function $F$ and $G$, respectively. Then, we can define the increasing map $T = G^{-1} \circ F$ such that $q = \Tpush p$, where $G^{-1}: [0,1] \to \RR$ is the quantile function of $q$:
\begin{align}
G^{-1}(u) := \inf\{ t: G(t) \geq u \}.
\end{align}
Indeed, it is well-known that $F(Z)\sim \mathrm{uniform}$ if $Z \sim p$ and $G^{-1}(U) \sim q$ if $U\sim \mathrm{uniform}$. \Cref{thm:tri} is a rigorous iteration of this univariate argument by repeatedly conditioning. Note that the increasing property is essential for claiming the uniqueness of $\Tb$. Indeed, for instance, let $p$ be standard normal, then both $\Tb = \id$ and $\Tb=-\id$ push $p$ to the same target normal density.
\label{exm:univ}
\end{example}

\begin{example}[Pushing uniform to normal]
	Let $p$ be uniform over $[0,1]$ and $q\sim\Nc(\mu, \sigma^2)$ be normal distributed. The unique increasing transformation 
	\begin{align}
	T(z) &= G^{-1} \circ F = \mu + \sqrt{2}\sigma \cdot \erf^{-1}(2z-1)\\
	&= \mu+\sqrt{2}\sigma\cdot\sum_{k=0}^\infty \frac{\pi^{k+1/2}c_k}{2k+1} (z-\tfrac{1}{2})^{2k+1},
	\end{align}
	where $\erf(t) = \frac{2}{\sqrt{\pi}}\int_0^t e^{-s^2} \rmd s$ is the error function, which was Taylor expanded in the last equality. The coefficients $c_0=1$ and $c_k = \sum_{m=0}^{k-1} \frac{c_m c_{k-1-m}}{(m+1)(2m+1)}$. We observe that the derivative of $T$ is an infinite sum of squares of polynomials. In particular, if we truncate at $k=0$, we obtain
	\begin{align}
	\label{eq:unif2norm}
	T(z) = \mu + \sqrt{2\pi}\sigma(z-\tfrac{1}{2}) + O(z^3).
	\end{align}
	\vspace{-2em}
	\label{exm:unif2norm}
\end{example}

\begin{example}[Pushing normal to uniform]
	Similar as above but we now find a map $S$ that pushes $q$ to $p$:
	\begin{align}
	S(x) &= F^{-1} \circ G = \Phi(\tfrac{x-\mu}{\sigma}) \\
	&= \frac{1}{2} + \frac{1}{\sqrt{\pi}} \sum_{k=0}^\infty \frac{(-1)^k}{k! (2k+1)}\left(\frac{x-\mu}{\sqrt{2}\sigma}\right)^{2k+1},
	\end{align}
	where $\Phi$ is the cdf of standard normal. As shown by \citet{Medvedev08}, $S$ must be the inverse of the map $T$ in \Cref{exm:unif2norm}. We observe that the derivative of $S$ is no longer a sum of squares of polynomials, but we prove later that it is approximately so. If we truncate at $k=0$, we obtain
	\begin{align}
	S(x) = \frac{1}{2} + \frac{1}{\sqrt{2\pi\sigma}}(x-\mu) + O(x^3),
	\end{align}
	where the leading term is also the inverse of the leading term of $T$ in \eqref{eq:unif2norm}.
\label{exm:norm2unif}
\end{example}

We end this section with two important remarks. 
\begin{remark}
If the target density has disjoint support \eg mixture of Gaussians (MoGs) with well-separated components, then the resulting transformation will need to admit sharp jumps for areas of near zero mass. This follows by analyzing the transformaiton 	$T(z) = G^{-1} \circ F$. The slope $T'(z)$ of $T(z)$ is the ratio of the quantile pdfs of the source density and the target density. Therefore, in regions of near zero mass for target density, the transformation will have near infinite slope. In \Cref{sec : mog trans}, we demonstrate this phenomena specifically for well-separated MoGs and show that a piece-wise linear function transforms a standard Gaussian to MoGs. This also opens the possibility to use the number of jumps of an estimated transformation as the indication of the number of components in the data density. 
\label{rem:disc}
\end{remark}

\begin{remark}
So far we have employed the (increasing) triangular map $\Tb$ \emph{explicitly} to represent our estimate of the target density. This is advantageous since it allows us to easily draw samples from the estimated density, and, if needed, it results in the estimated density formula \eqref{eq:pf} immediately. An alternative would be to parameterize the estimated density directly and explicitly, such as in mixture models, probabilistic graphic models and sigmoid belief networks. The two approaches are conceptually \emph{equivalent}: Thanks to \Cref{thm:tri}, we know choosing a family of triangular maps fixes a family of target densities that we can represent, and conversely, choosing a family of target densities fixes a family of triangular maps that we can \emph{implicitly} learn. The advantage of the former approach is that given a sample from the target density, we can infer the ``pre-image'' in the source domain while this information is lost in the second approach.
\label{rem:equiv}
\end{remark}
\vspace{-2em}
\section{Connection to existing works}
\label{sec:prev}
The results in \Cref{sec:bg} suggest using \eqref{eq:ML} with $\Fc$ being a class of triangular maps for estimating a probability density $q$. In this section we put this general approach into historical perspective, and connect it to the many recent works on generative modelling. Due to space constraint, we limit our discussion to work that are directly relevant to ours.

{\bf Origins of triangular map}
: \citet{Rosenblatt52}, among his contemporary peers, used the triangular map to transform a continuous multivariate distribution into the uniform distribution over the cube. Independently, \citet{Knothe57} devised the triangular map to transform uniform distributions over convex bodies and to prove generalizations of the Brunn-Minkowski inequality. \citet{Talagrand96}, unaware of the previous two results and in the process of proving some sharp Gaussian concentration inequality, effectively discovered the triangular map that transforms the Gaussian distribution into any continuous distribution. The work of \citet{BogachevKM05} rigorously established the existence and uniqueness of the triangular map and systematically studied some of its key properties. \citet{CarlierGS10} showed surprisingly that the triangular map is the limit of solutions to a class of Monge-Kantorovich mass transportation problems under quadratic costs with diminishing weights. None of these pioneering works considered using triangular maps for density estimation.

{\bf Iterative Gaussianization and Normalizing Flow}
:
In his seminal work, \citet{Huber85} developed the important notion of non-Gaussianality to explain the projection pursuit algorithm of \citet{FriedmanSS84}. Later, \citet{ChenGopinath01}, based on a heuristic argument, discovered the triangular map approach for density estimation but deemed it impractical because of the seemingly impossible task of estimating too many conditional densities. Instead, \citet{ChenGopinath01} proposed the iterative Gaussianization technique, which essentially decomposes\footnote{This can be made precise, much in the same way as decomposing a triangular matrix into the product of two rotation matrices and a diagonal matrix, \ie the so-called Schur decomposition.} the triangular map into the composition of a sequence of alternating diagonal maps $\Db_t$ and linear maps $\Lb_t$. The diagonal maps are estimated using the univariate transform in \Cref{exm:univ} where $G$ is standard normal and $F$ is a mixture of standard normals. Later, \citet{LaparraCVM11} simplified the linear map into random rotations. Both approaches, however, suffer cubic complexity \wrt dimension due to generating or evaluating the linear map. The recent work of \citep{TabakVE10,TabakTurner13} coined the name \emph{normalizing flow} and further exploited the straightforward but crucial observation that we can approximate the triangular map through a sequence of ``simple'' maps such as radial basis functions or rotations composed with diagonal maps. Similar simple maps have also been explored in \citet{BalleLS16}. \citet{RezendeMohamed15} designed a ``rank-1'' (or radial) normalizing flow and applied it to variational inference, largely popularizing the idea in generative modelling. These approaches are not estimating a triangular map \emph{per se}, but the main ideas are nevertheless similar.

\begin{table*}[ht]
\caption{Various auto-regressive and flow-based methods expressed under a unified framework. All the conditioners can take inputs $\xv$ instead of $\zv$. The symbol \faShareAlt \; is used for weight sharing, $\mathghost$ for use of masks for efficient implementation, \faUniversity\ for universality of the method and, $\Delta$ if the method learns a triangular transformation explicitly (E) or implicitly (I). $?$ implies that universality of these methods has neither been proved or disproved although it can now be analyzed with ease using our framework. $S_j(z_j ; \thetav_j)$ is defined in \cref{eq:cond} and $\Pf_{2r+1}(z_j; \av_j)$ is defined in \cref{eq:sos}. }
\vspace{1em}
\centering
\begin{tabular}{c|c|c|c|c|c|c}
\toprule
 Model & conditioner $C_j$ output & $T_j\big(z_j~; C_j(z_1, \ldots, z_{j-1})\big)$ & \faShareAlt & $\mathghost$ & \faUniversity & $\Delta$ \\[3pt]
 \midrule
 Mixture \citep[\eg][]{McLachlanPeel04} & $\thetav_j$ & $S_j(z_j; \thetav_j)$ &\xmark & \xmark & \cmark & I \\
 \citep{BengioBengio99} & $\thetav_j(z_{<j})$ & $S_j(z_j; \thetav_j)$ & \xmark & \xmark &? & I \\
  MADE \citep{GermainGML15} & $\thetav_j(z_{<j})$ & $S_j(z_j; \thetav_j)$ & \cmark & \cmark &? & I \\[3pt]
 NICE \citep{DinhKB15} &  $\mu_j(z_{<l})$& $z_j + \mu_j\cdot \mathbf{1}_{j \not\in [l]}$  & \xmark & \xmark & ? & E \\[3pt]
 NADE \citep{UriaCGML16} & $\thetav_j(z_{<j})$ & $S_j(z_j; \thetav_j)$ & \cmark & \xmark & ? & I \\[3pt]
IAF \citep{KingmaSJCSW16} & $\sigma_j(z_{<j}), ~\mu_j(z_{<j})$  & $\sigma_j z_j + (1-\sigma_j)\mu_j$ & \cmark & \cmark & ? & E \\[3pt]
 MAF \citep{PapamakariosPM17} & $\alpha_j(z_{<j}), ~\mu_j(z_{<j})$  & $z_j \exp(\alpha_j) + \mu_j$  & \cmark & \cmark & ? & E \\[3pt]
 Real-NVP \citep{DinhSDB17} & $\alpha_j(z_{<l})$, $\mu_j(z_{<l})$ & $\exp(\alpha_j \cdot  \mathbf{1}_{j \not\in [l]}) \cdot  z_j  + \mu_j \cdot  \mathbf{1}_{j \not\in [l]}$ & \xmark & \xmark & ? & E \\[3pt]
 NAF \citep{HuangKLC18} & $\wv_j(z_{<j})$ & DNN($z_j~; \wv_j$) & \cmark & \cmark & \cmark & E \\[3pt]
 \midrule
 SOS & $\av_j(z_{<j})$ & $\Pf_{2r+1}(z_j; \av_j) $ & \cmark & \cmark & \cmark & E \\[3pt]
 \bottomrule
\end{tabular}
\label{tab: flows}
\end{table*}

(\emph{Bona fide}) {\bf Triangular Approach}:
\citet{DecoBrauer95} (see also \citet{Redlich93}), to our best knowledge, is among the first to mention the name ``triangular'' \emph{explicitly} in tasks  (nonlinear independent component analysis) related to density estimation. More recently, \citet{DinhKB15} recognized the promise of even simple triangular maps in density estimation. The (increasing) triangular map in \citep{DinhKB15} consists of two simple (block) components: $T_1(\xv_1) = \xv_1$ and $T_2(\xv_1, \xv_2) = \xv_2 + m(\xv_1)$, where $\xv = (\xv_1, \xv_2)$ is a two-block partition and $m$ is a map parameterized by a neural net. The advantage of this triangular map is its computational convenience: its Jacobian is trivially 1 and its inversion only requires evaluating $m$. \citet{DinhKB15} applied different partitions of variables, iteratively composed several such simple triangular maps and combined with a diagonal linear map\footnote{They also considered a more general coupling that may no longer be triangular.}. However, these triangular maps appear to be too simple and it is not clear if through composition they can approximate any increasing triangular map. In subsequent work, \citet{DinhSDB17} proposed the extension where $T_1(\xv_1) = \xv_1$ but $T_2(\xv_1, \xv_2) = \xv_2 \odot \exp(s(\xv_1)) + m(\xv_1)$, where $\odot$ denotes the element-wise product. This map is again increasing triangular. \citet{MoselhyMarzouk12} employed triangular maps for Bayesian posterior inference, which was further extended in \citep{MarzoukMPS16} for sampling from an (unknown) target density. One of their formulations is  essentially the same as our \cref{eq:KL}.

{\bf Autoregressive Neural Models}: A joint probability density function can be factorized into the product of marginal and conditionals: 
\begin{align}
\textstyle
q(x_1, \ldots, x_d) = q(x_1) \prod_{j=2}^d q(x_j | x_{j-1}, \ldots, x_1).
\end{align}
In his seminal work, \citet{Neal92} proposed to model each (discrete) conditional density by a simple linear logistic function (with the conditioned variables as inputs). This was later extended by \citet{BengioBengio99} using a two-layer nonlinear neural net. The recent work of \citet{UriaCGML16} proposed to decouple the hidden layers in \citet{BengioBengio99} and to introduce heavy weight sharing to reduce overfitting and computational complexity. Already in \citep{BengioBengio99}, univariate mixture models were mentioned as a possibility to model each conditional density, which was further substantiated in \citep{UriaCGML16}.
More precisely, they model the $j$-th conditional density as:
\begin{align}
\label{eq:mix}
q(x_j|x_{j-1}, \ldots, x_1) = \sum_{\kappa=1}^k w_{j, \kappa}~ \Nc(x_j; \mu_{j, \kappa}, \sigma_{j,\kappa}) \\
\thetav_j := (w_{j,\kappa}, \mu_{j,\kappa}, \sigma_{j,\kappa})_{\kappa=1}^k = C_{j}(x_{j-1}, \ldots, x_1),
\end{align}
where $C_j$ is the so-called conditioner network that outputs the parameters for the (univariate) mixture distribution in \eqref{eq:mix}. According to \Cref{exm:univ} there exists a unique increasing map $S_j(\cdot~; \thetav_j)$ that maps a univariate standard normal random variable $z_j$ into $x_j$ that follows \eqref{eq:mix}. In other words, 
\begin{align}
\label{eq:cond}
x_j = S_j(z_j; \thetav_j) = : T_j(z_1, \ldots, z_{j-1}, z_j),
\end{align}
where the last equality follows from induction, using the fact that $\thetav_j = C_j(x_{j-1}, \ldots, x_1)$. Thus, as already pointed out in \Cref{rem:equiv}, specifying a family of conditional densities as in \eqref{eq:mix} is \textbf{equivalent} as (implicitly) specifying a family of triangular maps. In particular, if we use a nonparametric family such as mixture of normals, then the induced triangular maps can approximate any increasing triangular map. The special case, when $k=1$ in \eqref{eq:mix}, was essentially dealt with by \citet{KingmaSJCSW16}: for $k=1$ the map $S_j(\thetav_j) = \mu_j + \sigma_j z_j$ hence the triangular map 
\begin{align}
\label{eq:affine}
T_j(z_1, \ldots, z_{j-1}, z_j) = \mu_j(z_{<j}) +\sigma_j(z_{<j}) \cdot z_j.
\end{align} 
Obviously, not every triangular map can be written in the form \eqref{eq:affine}, which is affine in $z_j$ when $z_{<j}$ are fixed. To address this issue, \citet{KingmaSJCSW16} composed several triangular maps in the form of \eqref{eq:affine}, hoping this suffices to approximate a generic triangular map. In contrast, \citet{HuangKLC18} proposed to replace the affine form in \eqref{eq:affine} with a univariate neural net (with $z_j$ as input and $\mu_j$ and $\sigma_j$ serve as weights). Lastly, based on binary masks, \citet{GermainGML15} and \citet{PapamakariosPM17} proposed efficient implementations of the above that compute all parameters in a single pass of the conditioner network. It should be clear now that
(a) autoregressive models implement exactly a triangular map;
(b) specifying the conditional densities directly is equivalent as specifying a triangular map explicitly.

{\bf Other Variants.} Recurrent nets have also been used in autoregressive models (effectively triangular maps). For instance, \citet{OOrdKK16} used LSTMs to directly specify the conditional densities while \citet{MacKayVBG18} chose to explicitly specify the triangular maps. The two approaches, as alluded above, are equivalent, although one may be more efficient in certain applications than the other. \citet{OlivaDZPSXS18} tried to combine both while \citet{KingmaDhariwal18} used an invertible $1\times 1$ convolution. We note that the work of \citet{OstrovskiDM18} models the conditional quantile function, which is equivalent to but can sometimes be more convenient than the conditional density. 

{\bf Non-Triangular Flows.} Sylvester Normalizing Flows (SNF) \cite{BergHTW18} and FFJORD \cite{GrathwohlCBSD18} are examples of normalizing flows that employ non-triangular maps. They both propose efficient methods to compute the Jacobian for change of variables. SNF utilizes Sylvester's determinant theorem for that purpose. FFJORD, on the other hand, defines a generative model based on continuous-time normalizing flows proposed by \citet{ChenRBD18}
and evaluates the log-density 
efficiently using Hutchinson's trace estimator.      
  
\vspace{-1em}
\section{Sum-of-Squares Polynomial Flow}
\label{sec:sos}

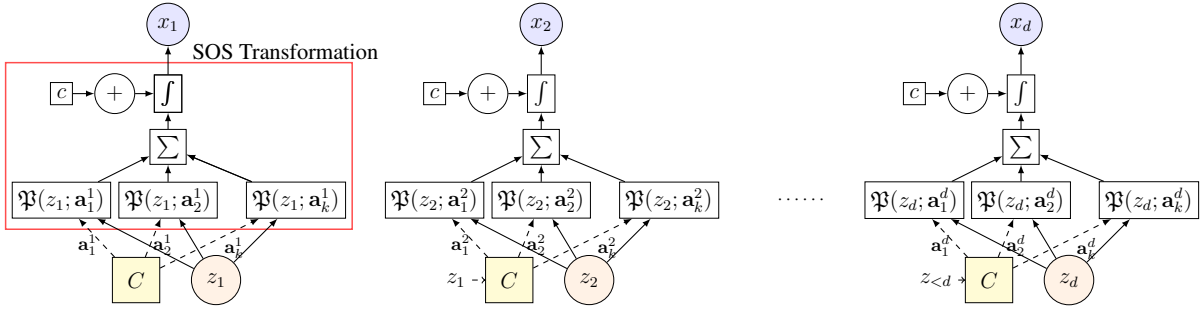
\begin{figure*}[t]
    \centering
    \resizebox{16cm}{!}{
    \begin{tikzpicture}
\node[draw, font=\large, fill=yellow!20, inner sep=0.8em] at (0,0) (z1){$C$};
\node[draw, font=\large] at ($(z1) - (1.4, -1.5)$) (p1){$\mathfrak{P}(z_1;\av_1^1)$};
\node[draw, font=\large] at ($(z1) - (-0.6, -1.5)$) (p2){$\mathfrak{P}(z_1;\av_2^1)$};
\node[draw, font=\large] at ($(z1) - (-3, -1.5)$) (pk){$\mathfrak{P}(z_1;\av_k^1)$};
\node[draw, shape=circle, font=\large, fill=orange!10, inner sep=0.5em] at ($(z1) - (-1.5,0)$) (c1){$z_1$};
\node[draw, font=\large] at ($(p2) - (0, -1)$) (sum){$\sum$};
\node[draw, font=\large] at ($(sum) - (0, -1)$) (int){$\int$};
\node[draw, font=\large] at ($(sum) - (0, -1)$) (int){$\int$};
\node[draw, shape=circle, font=\large] at ($(int) - (1, 0)$) (cons){$+$};
\node[draw, font=\large] at ($(cons) - (1, 0)$) (l){$c$};
\node[draw, shape=circle, fill = blue!10, font=\large] at ($(int) - (0, -1.3)$) (x1){$x_1$};
\node[draw=none, font=\large] at ($(x1) - (-2.2,0.5)$) (sosos){SOS Transformation};
\node[draw, red!70, fit=(p1)(p2)(pk)(sum)(int)(cons)(l), inner ysep = 2mm, thick](box){};
\draw[->, dashed, line width=0.2mm, -latex](z1)--(p1)node[near end, below]{$\av_1^1$};
\draw[->, dashed, line width=0.2mm, -latex](z1)--(p2)node[near end, below, xshift=1.1mm]{$\av_2^1$};
\draw[->, dashed, line width=0.2mm, -latex](z1)--(pk)node[near end, below]{$\av_k^1$};
\draw[->, line width=0.2mm, -latex](c1)--(p1);
\draw[->,line width=0.2mm, -latex](c1)--(p2);
\draw[->, line width=0.2mm, -latex](c1)--(pk);
\draw[->, line width=0.2mm, -latex](p1)--(sum);
\draw[->, line width=0.2mm, -latex](p2)--(sum);
\draw[->, line width=0.2mm, -latex](pk)--(sum);
\draw[->, line width=0.2mm, -latex](pk)--(sum);
\draw[->, line width=0.2mm, -latex](sum)--(int);
\draw[->, line width=0.2mm, -latex](int)--(x1);
\draw[->, line width=0.2mm, -latex](l)--(cons);
\draw[->, line width=0.2mm, -latex](cons)--(int);

\node[draw, font=\large, fill=yellow!20, inner sep=0.8em] at ($(z1) - (-7,0)$) (c2){$C$};
\node[draw, font=\large] at ($(c2) - (1.4, -1.5)$) (p12){$\mathfrak{P}(z_2;\av_1^2)$};
\node[draw, font=\large] at ($(c2) - (-0.6, -1.5)$) (p22){$\mathfrak{P}(z_2;\av_2^2)$};
\node[draw, font=\large] at ($(c2) - (-3, -1.5)$) (pk2){$\mathfrak{P}(z_2;\av_k^2)$};
\node[draw, shape=circle, font=\large, fill=orange!10, inner sep=0.5em] at ($(c2) - (-1.5,0)$) (z2){$z_2$};
\node[draw, font=\large] at ($(p22) - (0, -1)$) (sum2){$\sum$};
\node[draw, font=\large] at ($(sum2) - (0, -1)$) (int2){$\int$};
\node[draw, shape=circle, font=\large] at ($(int2) - (1, 0)$) (cons2){$+$};
\node[draw, font=\large] at ($(cons2) - (1, 0)$) (l2){$c$};
\node[draw, shape=circle, fill = blue!10, font=\large] at ($(int2) - (0, -1.3)$) (x2){$x_2$};
\node[draw=none, font=\large] at ($(c2) - (1, 0)$) (z11){$z_1$};

\draw[->, dashed, line width=0.2mm, -latex](c2)--(p12)node[near end, below]{$\av_1^2$};
\draw[->, dashed, line width=0.2mm, -latex](c2)--(p22)node[near end, below, xshift=1.2mm]{$\av_2^2$};
\draw[->, dashed, line width=0.2mm, -latex](c2)--(pk2)node[near end, below]{$\av_k^2$};
\draw[->, line width=0.2mm, -latex](z2)--(p12);
\draw[->,line width=0.2mm, -latex](z2)--(p22);
\draw[->, line width=0.2mm, -latex](z2)--(pk2);
\draw[->, line width=0.2mm, -latex](p12)--(sum2);
\draw[->, line width=0.2mm, -latex](p22)--(sum2);
\draw[->, line width=0.2mm, -latex](pk2)--(sum2);
\draw[->, line width=0.2mm, -latex](sum2)--(int2);
\draw[->, line width=0.2mm, -latex](int2)--(x2);
\draw[->, line width=0.2mm, -latex](l2)--(cons2);
\draw[->, line width=0.2mm, -latex](cons2)--(int2);
\draw[->, dashed] (z11) -- (c2);

\node[draw, font=\large, fill=yellow!20, inner sep=0.8em] at ($(c2) - (-9,0)$) (c3){$C$};
\node[draw=none] at ($(c2) - (-5.5,-1.5)$){$\cdots\cdots$};
\node[draw, font=\large] at ($(c3) - (1.4, -1.5)$) (p13){$\mathfrak{P}(z_d;\av_1^d)$};
\node[draw, font=\large] at ($(c3) - (-0.6, -1.5)$) (p23){$\mathfrak{P}(z_d;\av_2^d)$};
\node[draw, font=\large] at ($(c3) - (-3, -1.5)$) (pk3){$\mathfrak{P}(z_d;\av_k^d)$};
\node[draw, shape=circle, font=\large, fill=orange!10, inner sep=0.5em] at ($(c3) - (-1.5,0)$) (z3){$z_d$};
\node[draw, font=\large] at ($(p23) - (0, -1)$) (sum3){$\sum$};
\node[draw, font=\large] at ($(sum3) - (0, -1)$) (int3){$\int$};
\node[draw, shape=circle, font=\large] at ($(int3) - (1, 0)$) (cons3){$+$};
\node[draw, font=\large] at ($(cons3) - (1, 0)$) (l3){$c$};
\node[draw, shape=circle, fill = blue!10, font=\large] at ($(int3) - (0, -1.3)$) (x3){$x_d$};

\draw[->, dashed, line width=0.2mm, -latex](c3)--(p13)node[near end, below]{$\av_1^d$};
\draw[->, dashed, line width=0.2mm, -latex](c3)--(p23)node[near end, below, xshift=1.2mm]{$\av_2^d$};
\draw[->, dashed, line width=0.2mm, -latex](c3)--(pk3)node[near end, below]{$\av_k^d$};
\draw[->, line width=0.2mm, -latex](z3)--(p13);
\draw[->,line width=0.2mm, -latex](z3)--(p23);
\draw[->, line width=0.2mm, -latex](z3)--(pk3);
\draw[->, line width=0.2mm, -latex](p13)--(sum3);
\draw[->, line width=0.2mm, -latex](p23)--(sum3);
\draw[->, line width=0.2mm, -latex](pk3)--(sum3);
\draw[->, line width=0.2mm, -latex](sum3)--(int3);
\draw[->, line width=0.2mm, -latex](int3)--(x3);
\draw[->, line width=0.2mm, -latex](l3)--(cons3);
\draw[->, line width=0.2mm, -latex](cons3)--(int3);
\node[draw=none, font=\large] at ($(c3) - (1, 0)$) (z1d){$z_{<d}$};
\draw[->, dashed] (z1d) -- (c3);
\end{tikzpicture}
    }
    \caption{Schematic of SOS flows depicting the conditioner network and relevant transformations. \Cref{fig:SOS_stack} shows the schematic for SOS Flows by stacking multiple blocks of SOS transformation.}
    \label{fig:SOS_single}
\end{figure*}

In \Cref{sec:bg} we developed a general framework for density estimation using triangular
maps, and in \Cref{sec:prev} we showed the many recent generative models are all trying to estimate a triangular map in one way or another. In this section we give a surprisingly simple way to parameterize triangular maps, which, when plugged into \eqref{eq:ML}, leads to a new density estimation algorithm that we call sum-of-squares (SOS) polynomial flow.

Our approach is motivated by some classical result on simulating univariate non-normal distributions. Let $z$ be univariate standard normal. \citet{Fleishman78} proposed to simulate a non-normal distribution by fitting a degree-3 polynomial:
\begin{align}
x = \Pf_3(z; \av) = a_0 + a_1 z + a_2 z^2 + a_3 z^3,
\end{align}
where the coefficients $\{a_l\}$ are estimated by matching the first 4 moments of $x$ with those of empirical data. This approach was quite popular in practice because it allows researchers to precisely control the moments (such as skewness and kurtosis). However, three difficulties remain: (1) with degree-3 polynomial one can only (approximately) simulate a (very) strict subset of non-normal distributions. This can be addressed by using polynomials of higher degrees and better quantile matching techniques \citep{Headrick09}. (2) The estimated coefficients $\{a_l\}$ may not guarantee the monotonicity of the polynomial, making inversion and density evaluation difficult, if not impossible. (3) Extension to the multivariate case was done through composing a linear map \citep{ValeMaurelli83}, which can be quite inefficient.

We show that all three difficulties can be overcome using SOS flows. First, let us recall a classic result in algebra:
\begin{theorem}
A univariate real polynomial is increasing iff it can be written as:
\begin{align}
\label{eq:sos}
\Pf_{2r+1}(z; \av) = c + \int_{0}^z \sum_{\kappa=1}^k \left( \sum_{l=0}^r a_{l, \kappa} u^l \right)^2 \rmd u,
\end{align}
where $c\in\RR$, $r\in \NN$, and $k$ can be chosen as small as 2.
\label{thm:sos}
\end{theorem}
Note that a univariate increasing polynomial is strictly increasing iff it is not a constant. \Cref{thm:sos} is obtained by integrating a nonnegative polynomial, which is necessarily a sum-of-squares, see \eg \citep{Marshall08}.
Now, by applying \eqref{eq:sos} to model each conditional density in \eqref{eq:cond} we effectively addressed the last two issues above. Pleasantly, this approach strictly generalizes the affine triangular map \eqref{eq:affine} of \citep{KingmaSJCSW16}, which amounts to truncating $r = 0$ in \eqref{eq:sos}. However, by using a larger $r$, we can learn certain densities more faithfully (especially for capturing higher order statistics), without significantly increasing the computational complexity. Additionally, implementing \eqref{eq:sos} in practice is simple: It can be computed \emph{exactly} since it is an integral of univariate polynomials.

Lastly, we prove that as the degree $r$ increases, we can approximate any triangular map. We prove our result for the domain $\Zsf = \Xsf = \RR^d$, but the same result holds for other domains if we slightly modify the proof in the appendix.
\begin{restatable}{theorem}{SOS}
Let $\Cc$ be the space of real univariate continuous functions, equipped with the topology of compact convergence. Then, the set of increasing polynomials is dense in the cone of increasing continuous functions.
\label{thm:univ}
\end{restatable}

Since the topology of pointwise convergence is weaker than that of compact convergence (\ie uniform convergence on every compact set), we immediately know that there exists a sequence of increasing polynomials of the form \eqref{eq:sos} that converges pointwise to any given continuous function. This universal property of increasing polynomials allows us to prove the universality of SOS flows, \ie the capability of approximating any (continuous) triangular map. 

SOS flow consists of two parts: an increasing (univariate) polynomial $\Pf_{2r+1}(z_j; \av_j)$ of the form \eqref{eq:sos} for modelling conditional densities and a conditioner network $C_j(z_1, \ldots, z_{j-1})$ for generating the coefficients $\av_j$ of the polynomial $\Pf_{2r+1}(z_j; \av_j)$.  In other words, the triangular map learned using SOS flows has the following form:
\begin{align}
\label{eq:SOS_T}
\forall j, ~~ T_j(z_1, \ldots, z_j) = \Pf_{2r+1}\big(z_j; C_j(z_1, \ldots, z_{j-1})\big).
\end{align}
If we choose a universal conditioner (that can approximate any continuous function), such as a neural net, then combining with \Cref{thm:sos} and \Cref{thm:univ} we verify that the triangular maps in the form of \eqref{eq:SOS_T} can approximate any increasing continuous triangular map in the pointwise manner. It then follows that the transformed densities will converge weakly to any desired target density (\ie in distribution). This solves the first issue mentioned before \Cref{thm:sos}.
We remark that our universality proof for SOS flows is significantly shorter and more streamlined than the previous attempt of \citet{HuangKLC18}, and it can be seemingly extended to analyze other models summarized in \Cref{tab: flows}.  

As pointed out by \citet{PapamakariosPM17} we can also construct conditioner networks $C_j$ that take inputs $x_1, \ldots, x_{j-1}$, instead of $z_1, \ldots, z_{j-1}$. They are equivalent in theory but one can be more convenient than the other, depending on the downstream application. \Cref{fig:SOS_single} illustrates the main components of a single-block SOS flow, where we implement the conditioner network in the same way as in \citep{PapamakariosPM17}. To get a higher degree approximation, we can either increase $r$ or stack a few single-block SOS flows, as shown in \Cref{fig:SOS_stack}. The former approach appears to be more general but also more difficult to train due to the larger number of parameters. Indeed, the effective number of parameters for SOS flows obtained by stacking $L$ blocks with $k$ polynomials of degree $2r+1$ is $L\cdot k \cdot (r+1)$ whereas achieving the same representation with a single block wide SOS flow would require $ 1/2 \cdot k \cdot ((2r+1)^L - 1)$ parameters. In \Cref{subsec : sim exp} we perform simulated experiments to compare deep \vs wide SOS flows.
\begin{figure}[t]
    \centering
    \resizebox{8cm}{!}{
    \begin{tikzpicture}
\node[draw, fill=yellow!20, font=\large, inner xsep=2em] at (0,0) (conditioner){Conditioner Network};
\node[draw, fill=purple!20 ] at ($(conditioner) - (2,1.5)$) (t1) {$\av_1$};
\node[draw, fill=purple!20 ] at ($(conditioner) - (1,1.5)$) (t2) {$\av_2$};
\node[draw, fill=purple!20 ] at ($(conditioner) - (0,1.5)$) (t3) {$\av_3$};
\node[draw, fill=purple!20 ] at ($(conditioner) - (-2,1.5)$) (tk) {$\av_k$};
\draw[-, dashed,  line width = 0.2mm] (t3.east) -- (tk.west);
\draw[->, line width=0.15mm, -latex] (conditioner.south)--(t1);
\draw[->, line width=0.15mm, -latex] (conditioner.south)--(t2);
\draw[->, line width=0.15mm, -latex] (conditioner.south)--(t3);
\draw[->, line width=0.15mm, -latex] (conditioner.south)--(tk);
\node[draw, fill=red!50, font=\large, inner xsep=5em] at ($(conditioner) - (0,3)$) (sos) {SOS Transformation};
\draw[->, line width=0.15mm, -latex] (t1)--(sos);
\draw[->, line width=0.15mm, -latex] (t2)--(sos);
\draw[->, line width=0.15mm, -latex] (t3)--(sos);
\draw[->, line width=0.15mm, -latex] (tk)--(sos);
\node[draw, fit=(conditioner)(t1)(t2)(t3)(tk)(sos), inner sep =1em](box){};
\node at (box.north)[above] {\textbf{Sum-of-Squares Polynomial Flows}};
\node[draw, fill=white, fill=orange!10, font = \large]  at ($(conditioner) - (4.5,0)$) (z) {$\mathbf{z}^0$};
\draw[->, dashed, line width=0.2mm, -latex](z)--(conditioner.west);
\node[draw, fill=white, fill=orange!10, font = \large] at ($(sos) - (-4.5,0)$) (x) {$\mathbf{z}^1$};
\draw[->, dashed, line width=0.2mm, -latex](sos.east)--(x);
\node[draw, inner ysep=5em, align=left] at ($(box) - (-6, 0)$) (box2){SOS \\ Flows};
\draw[->, dashed, line width=0.2mm, -latex](x)-- (x -| box2.west);
\node[draw=none] at ($(box) - (-7.5, 0)$){$\cdots\cdots$};
\node[draw, inner ysep=5em, align=left] at ($(box) - (-9, 0)$) (box3){SOS \\ Flows};
\node[draw, fill=blue!10, font=\large] at ($(sos) - (-10.5, 0)$) (zf){$\mathbf{x}$};
\draw[-, dashed, line width =0.15mm, -latex](zf -| box3.east) -- (zf);
\end{tikzpicture}
    }
    \caption{Schematic of SOS flows by stacking multiple blocks of SOS transformation.}
    \label{fig:SOS_stack}
\end{figure}
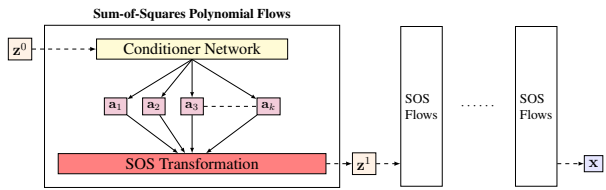

SOS flow is similar to the neural autoregressive flow (NAF) of \citet{HuangKLC18} in the sense that both are capable of approximating any (continuous) triangular map hence learning any desired target density. However, SOS flow has the following advantages: 
\begin{itemize}[itemsep=1pt,topsep=1pt]
    \item As mentioned before, SOS flow is a strict generalization of the inverse autoregressive flow (IAF) of \citet{KingmaSJCSW16}, which corresponds to setting $r=0$. 
    \item SOS flow is more interpretable, in the sense that its parameters (\ie coefficients of the polynomials) directly control the first few moments of the target density.
    \item SOS flow may be easier to train, as there is no constraint on its parameters $\av$. In contrast, NAF needs to make sure the parameters are nonnegative\footnote{A typical remedy is to re-parameterize through an exponential transform, which, however,  may lead to overflows or underflows.}.
\end{itemize}
\section{Experiments}
\label{sec:exp}

\begin{figure*}[t]
    \centering
    \includegraphics[height=4cm, keepaspectratio]{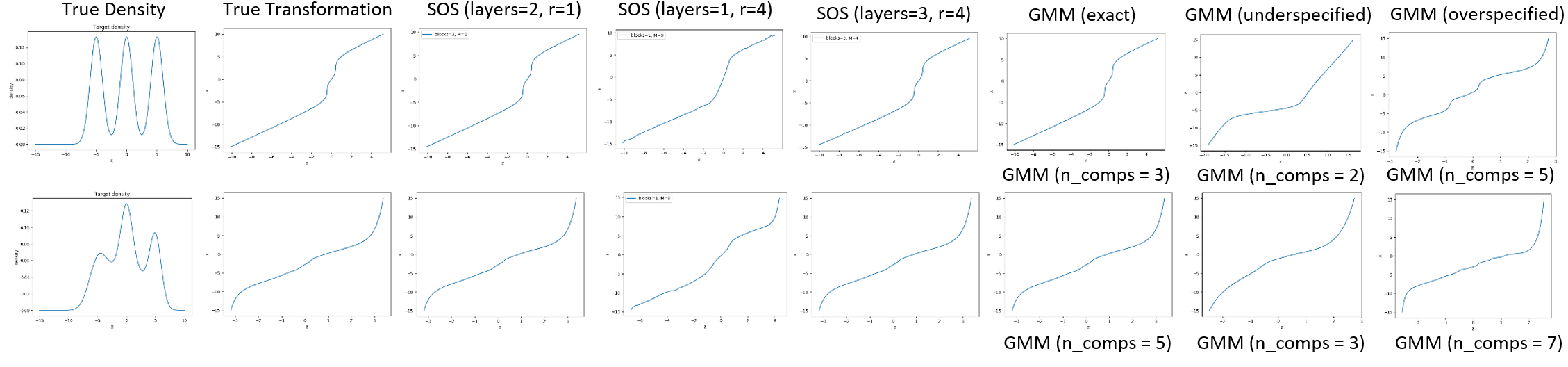}
    \vspace{-3ex}
    
    \caption{\textbf{Top Row:} First plot from the left shows the target density, a mixture of three component Gaussians with means = (-5, 0, 5), variances = (1, 1, 1) and, weights = (1/3, 1/3, 1/3). The second plot shows the exact transformation required to transform a standard Gaussian to this mixture. The next three plots shows the transformation learned by SOS flows with different configurations (deep, wide and wide-deep, respectively). The last three plots show the transformation learned by estimating the parameters of the Gaussian mixture using log-likelihood with exact (3), under-specified (2) and over-specified (5) number of components respectively.  \textbf{Bottom Row:} Same as Top Row but with target density being a mixture of five Gaussians with means = (-5, -2, 0, 2, 5), variances = (1.5, 2, 1, 2, 1) and, weights = 0.2 each.}
    \vspace{-2ex}
    \label{fig:gaussian conditionals}
\end{figure*}

We evaluated the performance of SOS flows on both synthetic and real-world datasets for density estimation, and compare it to several alternative autoregressive models and flow based methods.  

\begin{table*}[t]
\begin{center}
\caption{Average test log-likelihoods and standard deviation for SOS flows over 10 trials (higher is better). The other methods report the average log-likelihood and standard deviation over five trials. The numbers in the parenthesis indicate the number of stacked blocks for the resultant transformation.}
\label{table : real-datasets}
\vspace{1ex}
\begin{tabular}{c|c|c|c|c|c} 
\toprule
Method & Power & Gas & Hepmass& MiniBoone & BSDS300\\
\midrule
 MADE  &0.40 $\pm$ 0.01    &8.47 $\pm$ 0.02   &-15.15 $\pm$ 0.02 & -12.24 $\pm$ 0.47 & 153.71 $\pm$ 0.28 \\
 MAF affine (5) & 0.14 $\pm$ 0.01  & 9.07 $\pm$ 0.02   &-17.70 $\pm$ 0.02 & -11.75 $\pm$ 0.44&155.69 $\pm$ 0.28\\
 MAF affine (10) &0.24 $\pm$ 0.01 & 10.08 $\pm$ 0.02&  -17.73 $\pm$ 0.02 & -12.24 $\pm$ 0.45 & 154.93 $\pm$ 0.28\\
 MAF MoG (5)  &0.30 $\pm$ 0.01   & 9.59 $\pm$ 0.02 & -17.39 $\pm$ 0.02 & -11.68 $\pm$ 0.44 & 156.36 $\pm$ 0.28\\
 TAN & 0.60 $\pm$ 0.01  & 12.06 $\pm$ 0.02 &-13.78 $\pm$ 0.02 & -11.01 $\pm$ 0.48 & 159.80 $\pm$ 0.07\\
 NAF DDSF (5)  &0.62 $\pm$ 0.01 & 11.91 $\pm$ 0.13 & -15.09 $\pm$ 0.40 & -8.86 $\pm$ 0.15 & 157.73 $\pm$ 0.04\\
 NAF DDSF (10)&0.60 $\pm$ 0.02  & 11.96 $\pm$ 0.33 & -15.32 $\pm$ 0.23 &  -9.01 $\pm$ 0.01 & 157.43 $\pm$ 0.30\\
 SOS  (7) & 0.60 $\pm$ 0.01 & 11.99 $\pm$ 0.41 & -15.15 $\pm$ 0.10 & -8.90 $\pm$ 0.11& 157.48 $\pm$ 0.41\\
 \bottomrule
\end{tabular}
\end{center}
\end{table*}

\begin{table}[t]
\begin{center}
\caption{Negative test log-likelihoods for various density estimation models on image datasets (lower is better). * results/models used multi-scale convolutional architectures.}
\label{table:img}
\vspace{1ex}
\begin{tabular}{c|c|c} 
\toprule
Method & MNIST & CIFAR10\\
\midrule
 Real-NVP & 1.06* & 3.49*\\
 Glow  & 1.05* & 3.35* \\
 FFJORD & 0.99* & 3.40*\\
  MADE  & 2.04 & 5.67\\
 MAF  & 1.89  & 4.31\\
 SOS   & 1.81 & 4.18 \\
 \bottomrule
\end{tabular}
\end{center}
\vspace{-4ex}
\end{table}

\subsection{Simulated Experiments}
\label{subsec : sim exp}
We performed a host of experiments on simulated data to gain in-depth understanding of SOS flows. 

In \Cref{fig:gaussian conditionals} we demonstrate the ability of SOS flows to represent transformations that lead to multi-modal densities by generating data from a mixture of Gaussians for two cases - well-connected and disjoint support. The true transformation can be computed exactly following \Cref{exm:univ}. We show three transformations learned by SOS flows for each case corresponding to a deep SOS flow, wide SOS flow and wide-deep SOS flow. As is evident, SOS flows were fairly successful in learning the transformations. We further estimated the parameters of these simulated densities using Gaussian mixtures trained using maximum likelihood under three cases - exact (same number of components as target density), under-specified (lesser number of components) and over-specified. Subsequently, we plot the resulting transformation in each case following \Cref{exm:univ}. While, GMMs with exact components work well as expected, the transformations learned by under-specified and over-specified models are not as good. This experiment also goes on to show that using a parameterized density to model conditionals is equivalent to implicitly learning a transformation.  
We also performed experiments to study the effect of relative ordering of variables for the conditioner network and the representational power of deep and wide SOS flows. Finally, we tested SOS flows on a suite of 2D simulated datasets -- Funnel, Banana, Square, Mixture of Gaussians and Mixture of Rings. Due to space constraints we defer the figures and explanations to \Cref{app: simulated exp}.

\subsection{Real-World Datasets}
\label{sec : real data}
We also performed density estimation experiments on 5 real world datasets that include four datasets from the UCI repository and BSDS300. These datasets have been previously considered for comparison of flows based methods \cite{HuangKLC18}.

The SOS transformation was trained using maximum likelihood method with source density as standard normal distribution. We used stochastic gradient descent to train our models with a batch size of 1000, learning rate = 0.001, number of stacked blocks = 8, number of polynomials ($k$) = 5 and, degree of polynomials ($r$) = 4 with number of epochs for training = 40. We compare our method to previous works on normalizing flows and autoregressive models which include MADE-MoG \cite{GermainGML15}, MAF \cite{PapamakariosPM17}, MAF-MoG \cite{PapamakariosPM17}, TAN \cite{OlivaDZPSXS18} and NAFs \cite{HuangKLC18}. In \Cref{table : real-datasets}, we report the average log-likelihood obtained using 10 fold cross-validation on held-out test sets for SOS flows. The performance reported for other methods are those reported in \citep{HuangKLC18}. The results show that SOS flows are able to achieve competitive performance as compared to other methods.

\section{Conclusion}
\label{sec:conclusion}
We presented a unified framework for estimating complex densities using monotone and bijective triangular maps. 
The main idea is to specify one-dimensional transformations and then iteratively extend to higher-dimensions using conditioner networks. Under this framework, we analyzed popular autoregressive and flow based methods, revealed their similarities and differences, and provided a unified and streamlined approach for understanding the representation power of these methods. Along the way we uncovered a new sum-of-squares polynomial flow that we show is universal, interpretable and easy to train. We discussed the various advantages of SOS flows for stochastic simulation and density estimation, and we performed various experiments on simulated data to explore the properties of SOS flows. Lastly, SOS flows achieved competitive results on real-world datasets. In the future we plan to carry out the analysis indicated in \Cref{tab: flows}, and to formally establish the respective advantages between deep and wide SOS flows.

\section*{Acknowledgement}
\ifrev
\else 
We thank the reviewers for their insightful comments, and Csaba Szepesv\'ari for bringing \citep{MulanskyNeamtu98} to our attention, which allowed us to reduce a lengthy proof of \Cref{thm:univ} to the current slim one. We would also like to thank Ilya Kostrikov for the code which we adapted for our SOS Flow implementation. We thank Junier Oliva for pointing out an oversight about TAN in a previous draft and Youssef Marzouk for bringing additional references to our attention. Finally, we gratefully acknowledge support from NSERC. PJ was also supported by the Cheriton Scholarship, Borealis AI Fellowship and Huawei Graduate Scholarship.
\fi
\bibliographystyle{icml2019}
\bibliography{dgm}

\newpage
\newpage
\clearpage

\twocolumn[
\icmltitle{Supplementary Material : Sum-of-Squares Polynomial Flow}
]
\appendix
\section{Simulated Experiments}
\label{app: simulated exp}

Here, we explore the effect of relative ordering for the conditioner network for SOS flows as well as mixture of Gaussians. We again generated two sets of 2D densities given by $p(x_1, x_2) = \Nc(x_2~; 0,4)\Nc(x_1~; 0.25x_2^2, 1)$ and $p(x_1, x_2) = \Nc(x_2~; 2,2)\Nc(x_1~; 1/3x_2^3, 1.5)$. However, we trained both SOS flows and GMMs with the reverse order \ie $x_1, x_2$. For SOS flows we again tested using both deep and wide flows whereas for MoGs we tested with varying number of components for each conditional. We present the plots in \Cref{fig:ordering experiment}. The best performance here is by a deep SOS flow. Furthermore, while a flat SOS flow is able to achieve almost the same geometrical shape as the target density, its learned density still differs from the true density. For mixture of Gaussians, a large number of components for each conditional improved the performance of the resulting model.  

\begin{figure}[h]
    \centering
    \includegraphics[width=8cm, keepaspectratio]{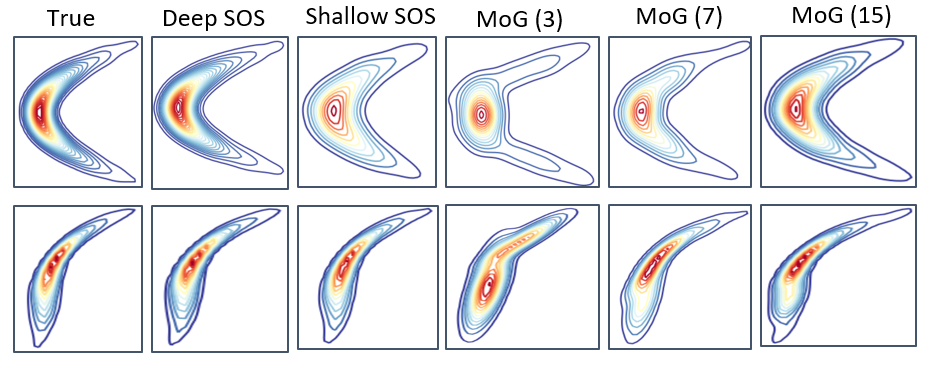}
    \caption{\textbf{Top:} Left plot shows the target density given by $p(x_1, x_2) = \Nc(x_2~; 0,4)\Nc(x_1~; 0.25x_2^2, 1)$. The second plot shows the density learnt by SOS flows with 3 blocks and a sum of 2 polynomials with degree 3 with ordering $(x_1,x_2)$. Third plot shows the density learnt by SOS flows with 1 block and a sum of 2 polynomials with degree 4 and ordering $(x_1, x_2)$. The last three plots  estimate this density using a Mixture of Gaussian conditionals with varying components given in parenthesis and ordering $(x_1, x_2)$. \textbf{Bottom:} Same as Top but with target density given by $p(x_1, x_2) = \Nc(x_2~;2,2)\Nc(x_1~;0.33x_1^3, 1.5)$.}
    \label{fig:ordering experiment}
\end{figure}

We also test the representational power of deep and wide SOS flows and the results are given in \Cref{fig:flat_deep}. In the first row, the true transformation was simulated by stacking multiple blocks of SOS transformation. Subsequently, we generated the target density using this transformation and estimated it using a deep flow, wide flow, wide-deep flow and mixture of Gaussians. In the second row, we simulated the true transformation using a single block SOS transformation and performed the same experiment as before. In both simulations, we tried to break our model by adding random noise to the coefficients of simulated transformation. As the figure shows, however, both deep and wide variants performed equally well in terms of representation. As expected however, the training time for wider flows was significantly longer than that for deeper flows.
\begin{figure}[t]
    \centering
    \includegraphics[width=8cm, keepaspectratio]{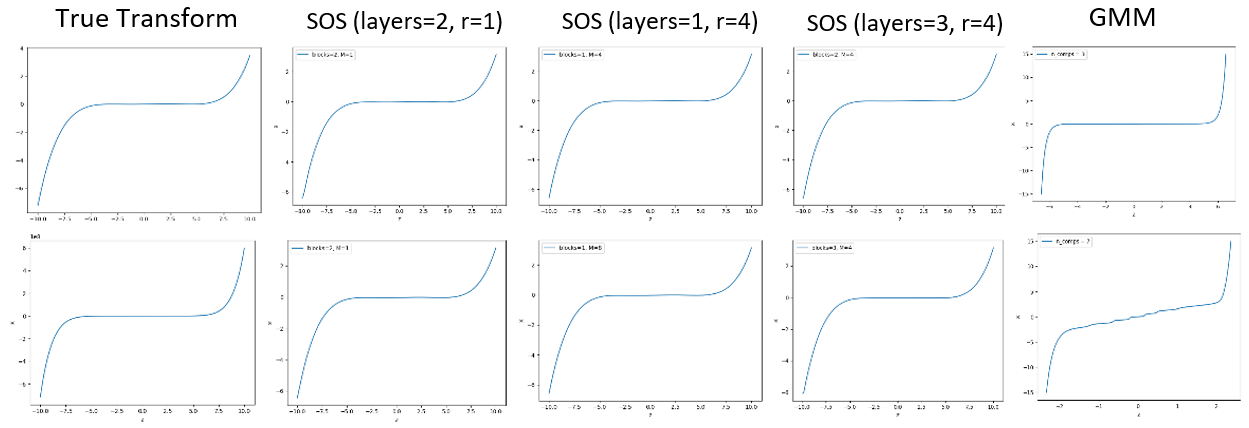}
    \caption{\textbf{Top Row:} Transformation defined by a deep SOS flow with $M=r=1$ and blocks =4. The next three plots show SOS flows learning this transformation with different configurations (deep, wide and, wide-deep). The last plot shows the transformation learned when a Gaussian mixture model learns the density (or transformation). \textbf{Bottom Row:} Same as Top Row but the true transformation was derived by a wide and shallow SOS flow with $M=r=4$ and blocks=1.}
    \label{fig:flat_deep}
\end{figure}

Finally, we tested SOS flows on a suite of 2D simulated datasets -- Funnel, Banana, Square, Mixture of Gaussians and Mixture of Rings. These datasets cover a broad range of geometries and have been considered before by \citet{WenliangSSG18}. For these experiments, we constructed our model by stacking 3 blocks with each block being a sum of \emph{two} polynomials each of degree \emph{four}. We plot the log density function learned by SOS flow and the true model in \Cref{fig:my_label}. The model is able to capture the true log density of datasets like Funnel and Banana. The true densities of Funnel and Banana are a simple linear transformation of Gaussians. Hence,  flow based models that learn a continuous and smooth transformation are expected to perform well on these datasets. However, SOS demonstrates certain artifacts at the sharp corners of the Square although it is able to capture the overall density nicely. These three datasets -- Funnel, Banana, and Square -- were part of the unimodal simulated datasets.

\begin{figure}[t]
    \centering
    \includegraphics[width=8cm]{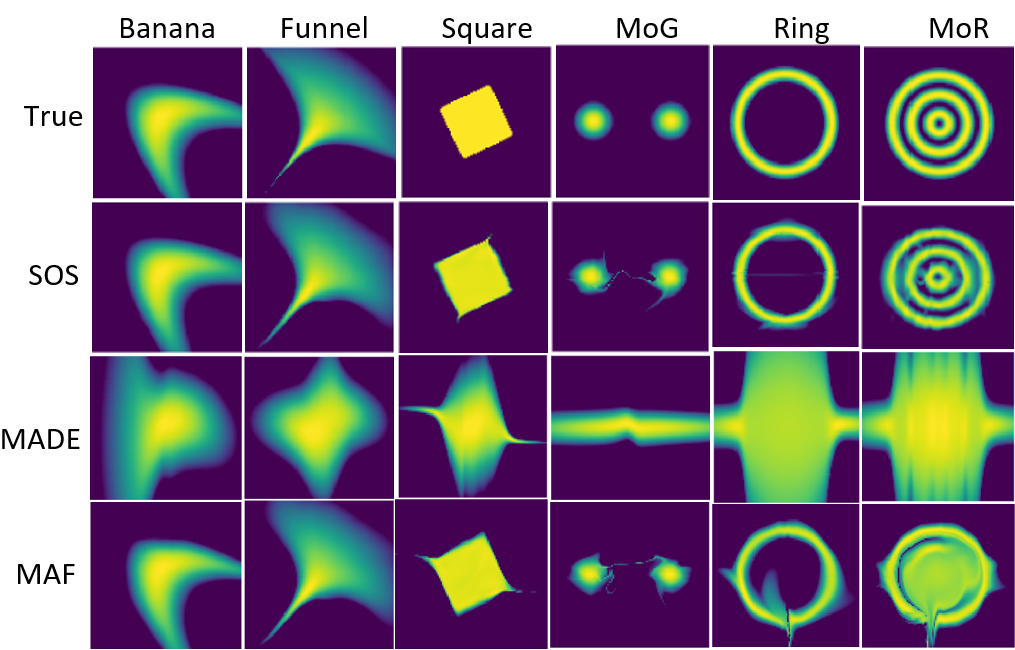}
    \caption{Log-densities for various toy-datasets. The top row shows the true log-densities. The next three rows are the log-densities for SOS flows, MADE and, MAF respectively.}
    \label{fig:my_label}
\vspace{-1em}
\end{figure}

The multimodal datasets included Mixture of Gaussians (MoG) and Mixture of Rings (MoR). As discussed earlier in \Cref{rem:disc}, when the target distribution has regions of near zero mass, the learned transformation admits sharp jumps to capture such regions. Flow based models by virtue of being invertible and smooth are often unable to learn such sharp jumps. SOS flows performs reasonably well for mixture of Gaussians although there are certain artifacts in the model that try to connect the two components. Similarly, there are some artifacts connecting the rings for the Mixture of Rings datasets. However, this issue of separated components can be dealt with relative ease in practice using clustering.   

\section{Transformation for Mixture of Gaussians}
\label{sec : mog trans}

\begin{figure*}[t]

\centering
\includegraphics[width=.3\textwidth]{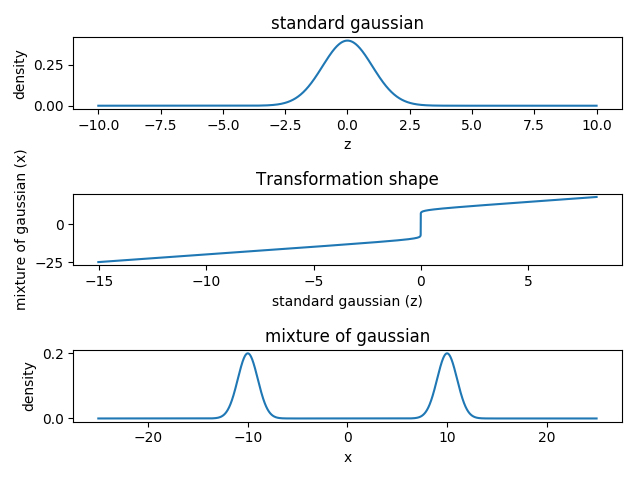}\hfill
\includegraphics[width=.3\textwidth]{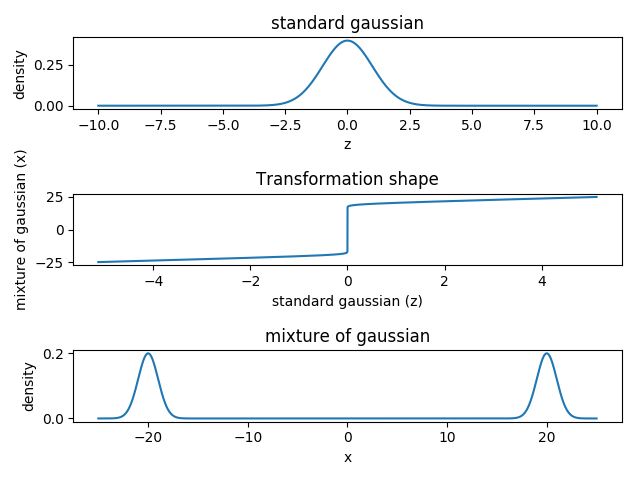}\hfill
\includegraphics[width=.3\textwidth]{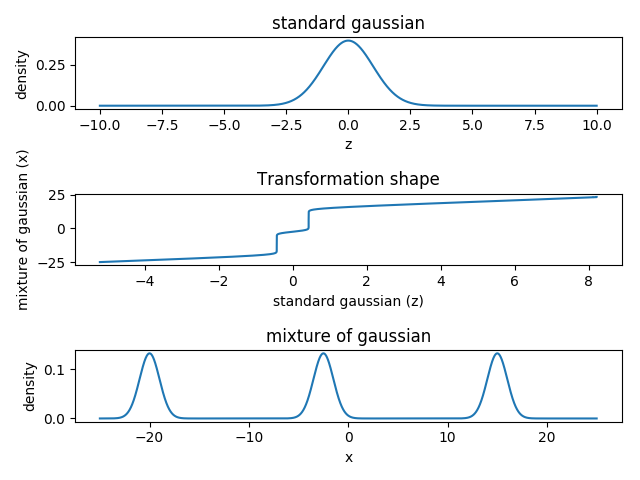}

\caption{Transformation curves from standard Gaussian to mixture of Gaussians.}
\label{fig:figure3}

\end{figure*}
The slope $T'(z)$ of $T$ at any point $z$ is given by
\begin{align}
    T'(z) &= \frac{p(z)}{q\Big(T(z)\Big)}, \quad \text{where} \quad x = T(z) \\
    &= \frac{p\Big(F^{-1}(t)\Big)}{q\Big(G^{-1}(t)\Big)}, \quad \text{where} \quad t = F(z)
\end{align}
 \ie the slope $T'(z)$ is the ratio of probability density quantiles (pdQs) of the source random variable and the target random variable.   

We now analyze the transformation required to transform a standard normal distribution to mixture of normal distributions. \Cref{fig:figure3} shows three columns of  plots: In the leftmost column, the top plot is the source distribution ($\Zsf$) which is standard normal. The bottom plot is the target distribution for the random variable $\Xsf$ which is a Gaussian mixture model with two components. The means are $-10$ and $10$, the variance is 1 and weights are 0.5 for each component respectively. The middle plot shows the transformation $T$ required to push forward a standard normal distribution to the target. In the second column of plots, we now transform a standard normal distribution to a mixture distribution but with means as -20 and 20, \ie the components are more separated. Finally, in the plots given in the rightmost column, we transform a standard normal distribution to a mixture of three Gaussians with means -20, -5, and 15. The variances are 1 and weights are $\frac{1}{3}$ respectively. 

We make the following observations here: In all three plots for the transformation, we notice that the transformation admits jumps (close to being vertical) \ie the slope at these points is large and close to infinity. This is expected since the regions where the target has almost zero mass but the source has finite mass would lead to a slope with such behavior. In the plots, this is the region in between the components where the mass of the target density approaches zero. Furthermore, the larger this area, the longer is the height of this jump (see plots on column one and column two). With densities that have two such areas, the transformation as expected has two jumps (plots on column three). The slope of $T$ on the extremes is a constant and is equal to the standard deviation of the component on that extreme. This is because:
\begin{align}
    \lim_{z\to \pm \infty} T'(z) = \lim_{z\to \pm\infty}\frac{p(z)}{q\Big(T(z)\Big)}
\end{align}
As $z \to \infty$, $q$ is approximately equal to the component on the positive extreme of the x-axis. This easily gives that $\lim_{z\to \infty} T'(z) = \sigma_{+}$ where $\sigma_{+}$ is the standard deviation of the component on the positive extreme of the x-axis; similarly, we get $\lim_{z\to -\infty} T'(z) = \sigma_{-}$ \ie $T'(z)$ is a constant in almost all the region of zero mass on the left of the component on the negative extreme and on the right of the positive extreme (verified in \Cref{fig:figure3}). Finally, the only regions where $T'(z)$ is finite is whenever $q(x) > q(\tilde{x})$ where $\tilde{x} \leq \mu_i \pm 2\sigma_i$ where the index $i$ stands for the $i^{th}$ component. Therefore, any $T$ that transforms a standard normal distribution to a mixture of Gaussians will be approximately piece-wise linear with jumps. The number of linear pieces in this transformation will be equal to the number of components in the mixture. The slopes of these linear pieces will be a function of the standard deviations of the mixture components. Additionally, the height of the jump will be a function of the mixing weights and standard deviation of the mixture components.

\section{Proofs}
\begin{lemma}[\citealt{MulanskyNeamtu98}]
Let $S$ be a dense subspace of $X$ and let $C\subseteq X$ be a convex set such that $\intr(C)\neq \emptyset$. Then $C\cap S$ is dense in $C$.
\end{lemma}
\begin{proof}
Since the interior $\intr(C)$ is open and nonempty, and $S$ is dense, we know $\intr(C) \cap S$ is dense in $\intr(C)$. (Every open set of $\intr(C)$ is also an open set of $X$, hence intersects the dense set $S$.) Moreover, since $C$ is convex and $\intr(C)\ne \emptyset$, we know $\cl(\intr(C)) = \cl(C)$, hence $\cl(\intr(C) \cap S) = \cl(C)$, \ie, $\intr(C)\cap S$, whence also the ``larger'' set $C\cap S$, is dense in $C$.
\end{proof}
\SOS*
\begin{proof}
Let us define $\Pc$ to be the space of polynomials, and $\Ic$ the space of increasing functions.
We need only prove on any compact set $K$, the set of polynomials of the form \eqref{eq:sos}, \ie $\Ic\cap\Pc$ thanks to \Cref{thm:sos}, is dense in $\Cc(K) \cap \Ic$. By Weierstrass' theorem we know $\Pc$ is dense in $\Cc(K)$. Moreover, the convex subset $\Ic \cap \Cc(K)$ has nonempty interior (take say a linear function with positive slope). Applying Lemma 1 above completes the proof.
\end{proof}

\end{document}